\documentclass[11pt]{paper}
\usepackage[utf8]{inputenc}
\pdfoutput=1
\usepackage[a4paper, total={6in, 8in}]{geometry}
\title{KLUCB Approach to Copeland Bandits}
\author{Nischal Agrawal, Prasanna Chaporkar}
\date{\today}

\usepackage[numbers]{natbib}
\usepackage{graphicx}
\usepackage{amsmath,amsfonts,amsthm,amssymb,bbm}
\usepackage{latexsym}
\usepackage{mathtools}
\usepackage{dsfont}
\usepackage{algorithm}
\usepackage{mwe}
\usepackage{subfig}
\usepackage{algpseudocode}
\newtheorem{theorem}{Theorem}

\newtheorem{lemma}{Lemma}[theorem]

\begin{document}

\maketitle
\begin{abstract}
Multi-armed bandit(MAB) problem is a reinforcement learning framework where an agent tries to maximise her profit by proper selection of actions through absolute feedback for each action. The dueling bandits problem is a variation of MAB problem in which an agent chooses a pair of actions and receives relative feedback for the chosen action pair. The dueling bandits problem is well suited for modelling a setting in which it is not possible to provide quantitative feedback for each action, but qualitative feedback for each action is preferred as in the case of human feedback. The dueling bandits have been successfully applied in applications such as online rank elicitation, information retrieval, search engine improvement and clinical online recommendation. We propose a new method called Sup-KLUCB for K-armed dueling bandit problem specifically Copeland bandit problem by converting it into standard MAB problem. Instead of using MAB algorithm independently for each action in a pair as in Sparring and in Self-Sparring algorithms, we combine a pair of action and use it as one action. Previous UCB algorithms such as Relative Upper Confidence Bound(RUCB) can be applied only in case of Condorcet dueling bandits, whereas this algorithm applies to general \textit{Copeland dueling bandits}, including \textit{Condorcet dueling bandits} as a special case. Our empirical results outperform state of the art Double Thompson Sampling(DTS) in case of Copeland dueling bandits. 
\end{abstract}

\section{Introduction}
A classic Multi Armed Bandit (MAB) problem is a reinforcement learning problem wherein an agent learns to play actions in order to maximise her profit. Initially agent is uninformed about any stochastic information about the actions. She learns to play through some feedback associated with previous actions. MAB problems possess dilemma of \textit{exploration} and \textit{exploitation}. Since true parameters are unknown, any algorithm can maintain only estimated parameters. Inadequate exploration might result in playing of sub-optimal actions which increase loss and excess exploration will lead to slow convergence which is also undesirable. For bandit problems, performance of any algorithm can be measured using \textit{cumulative regret}. In MAB problems regret for any action played at a time instant is defined as the gap between expected reward of best action and expected reward of current action. Thus, each algorithm will try to minimise cumulative regret. MAB problems have very wide range of application. They have been successfully applied in fields of online advertisements, clinical trials, adaptive routing and communication systems. 
\\
The dueling bandit problem \citep{Yue2012} is a variation of MAB problem in which an agent chooses a pair of actions and receives relative feedback about preference of action in chosen pair. Unlike MAB problem in which agent receives quantitative feedback for her actions, only qualitative feedback is received in dueling bandit problem. Study of these kinds of problems is important when dealing with feedback which are naturally relative (e.g. feedback given by humans) or it is inefficient to provide absolute feedback. The dueling bandits have been successfully applied in applications such as information retrieval \citep{yue2009interactively}, search engine improvement, clinical online recommendation \citep{sui2014clinical} and online rank elicitation \citep{szorenyi2015online}. 
\\
In a k-armed dueling bandit problem, when an action defeats rest of the actions, it is called \textit{Condorcet winner}. But this might not always be the case e.g. the best football team might not defeat rest of the teams. In absence of Condorcet winner, there can be several other criteria for judging the winner. \textit{Copeland winner} \citep{ZoghiCopeland} is the action which defeats maximum number of other actions. \textit{Borda Winner} \citep{Jamieson2015} is the winner with largest \textit{Borda score} defined by $\frac{1}{K}\sum_j^K P_{ij}$, where $P_{ij}$ is probability of action $i$ defeating action $j$. We will assume unique winner throughout our discussion. We will refer to an action as an arm hereafter in our discussion.

Our paper is organised as follows. In following sub-section we discuss related works. In section 2, we formally define MAB problem and k-armed bandit problem. In section 3, we provide detailed description of Sup-KLUCB algorithm and key intuitions behind it. 
Section 4, we present our results and comparison with various algorithms. We conclude the report in section 5.

\subsection{Related Work}
Standard MAB problem have been studied quite extensively in past. Most notable work has been done by Lai and Robbins \citep{lai1985} where they developed asymptotic lower bounds for regret to be of order $\mathcal{O}(\log(n))$. Algorithms which follow above rules are called \textit{uniformly good} and are asymptotically efficient. Graves and Lai \citep{graves1997asymptotically} proved the bounds by applying bandits in a controlled Markov chain setting. Various algorithms with varying success have been put forth to solve MAB problem with the most important ones as Upper Confidence Bound (UCB) \citep{auer2002finite} and Kullback-Leibler UCB (KL-UCB) \citep{garivier2011kl}. \\
KL-UCB algorithm is an online, horizon free index policy for stochastic bandit problems. Horizon is the number of steps told in advance to any algorithm before which it has to produce single best arm and henceforth continue to exploit that arm. Horizon free algorithms do not require any specified horizon and evaluation process continuous indefinitely. Thus horizon free algorithm have to minimise regret across all horizons by rapidly converging to selection of optimal arm. Authors show that regret of KL-UCB algorithm is upper bounded by 
\begin{equation}
    \underset{n \rightarrow \infty}{\lim \sup} \frac{ \mathds{E}[R_n]}{ \log(n)} \leq \sum\limits_{a:\mu_a < \mu_{a^*}} \frac{\mu_a - \mu_a^*}{\mathcal{K}(\mu_a , \mu_{a^*})}
\end{equation}
where $\mathcal{K}(p,q) = p\log(\tfrac{p}{q}) + (1-p)\log(\tfrac{1-p}{1-q})$ denotes the Kullback-Leibler divergence between Bernoulli distributions of parameter $p$ and $q$ respectively. $\mu_a \in [0,1]$ is expected reward for action $a$ and ${a^*}$ is the action with highest expected reward. Authors also show a non-asymptotic upper bound on number of draws of sub optimal arm $a$: for all $\epsilon > 0 $ there exists $C_1, C_2(\epsilon)$ and $\beta(\epsilon)$ such that
\begin{equation}
    \mathds{E}[N_n(a)] \leq \frac{\log(n)}{\mathcal{K}(\mu_a , \mu_a^*)}(1+\epsilon) + C_1 \log(\log(n)) + \frac{C_2(\epsilon)}{n^{\beta(\epsilon)}}.
\end{equation}
\\
Several algorithms have been proposed for k-armed dueling bandit problem. We can briefly categorise them into 2 categories viz. Asymmetric and Symmetric \citep{sui2018advancements}. Asymmetric type of algorithms consider both arms independently. Usually it selects first arm which has best performance (exploitation arm) and then it selects second arm to duel against the first arm with aim to identify an arm which can outperform the first (exploration arm). Interleaved Filter (IF), Beat the Mean (BtM), SAVAGE, Doubler, Relative Upper Confidence Bound (RUCB) \citep{Zoghi2014}, MergeRUCB and Double Thompson Sampling (DTS) \citep{wu2016double} are few examples of asymmetric type of algorithms. Symmetric type of algorithms treat the choice of two arms symmetrically. Sparring and Self-Sparring algorithms are few examples of symmetric algorithms. \\
RUCB algorithm \citep{Zoghi2014} extends UCB to dueling bandit problems by using a upper confidence bound on preference probabilities. Cumulative regret of RUCB after $T$ time steps is bounded by $\mathcal{O}(K\log(T))$. Cumulative regret after $t$ iterations (for some $\alpha > 1$), is bounded by
\begin{equation}
    \mathds{E}[R_t] \leq \Delta^*\Big(\frac{(4\alpha -1)K^2}{2\alpha -1 }\Big)^{\tfrac{1}{2\alpha -1}} \frac{2\alpha -1}{2\alpha -2} + \sum\limits_{i>j} 2\alpha \frac{\Delta_i + \Delta_j}{\min\{\Delta_i^2, \Delta_j^2\}}\ln(t),
\end{equation}
where $\Delta^* := \max_i \Delta_i$ and $\Delta_{i} = P_{ai} - 0.5$ where $a$ is the best arm. Double Thompson Sampling (DTS) \citep{wu2016double} as the name suggests uses Thompson Sampling twice, once it is used to break ties while selecting first arm in RUCB and then it used to sample second arm. Cumulative regret of DTS is bound by $\mathcal{O}(K\log(T) + K^2 \log(\log(T))$. DTS is state of art algorithm for small scale dueling bandits whereas MergeRUCB (variant of RUCB) is state of art algorithm for large scale bandits. However the scope of RUCB type of algorithms are limited to Condorcet type problems whereas DTS extends to Copeland case as well.\\
Our algorithm reduces dueling bandit problem to standard MAB problem. Previously also, algorithms which converts dueling bandit problem to conventional MAB problem have been proposed. Doubler \citep{ailon2014reducing} is first approach in this direction. Doubler assumes that probability of an arm winning over another is the linear function of the underlying utility of each arm. This utility association assumption requires total ordering of arms. Sparring \citep{ailon2014reducing} algorithm also assumes total ordering as in Doubler. Sparring algorithm uses separate MAB algorithms to choose different arms, reducing dueling bandit problem into two MAB problems which can be related to \textit{adversarial bandit} problem. Self-Sparring \citep{sui2017multi} performs better than Sparring algorithm and it is upper bounded by $\mathcal{O}(K\log(T))$. Self-Sparring independently chooses arm by calling stochastic MAB algorithms like Thompson sampling as a subroutine. Self-Sparring can used $m$ independent MAB to duel $m$ arms simultaneously and dueling bandit problem is a special case with $m=2$.

\subsection{Our Contribution}
We propose an algorithm called Sup-KLUCB to solve k-armed dueling bandit problem specially Copeland bandit problem. Sup-KLUCB is horizon free, stochastic reinforcement learning algorithm like KLUCB. Unique feature of Sup-KLUCB algorithm is its flexibility, with minor changes in objective function, it can be used to solve various type of dueling bandit problems which has single unique winner such as Copeland, Condorcet(special case of Copeland problem), Borda etc. In this paper we focus on Copeland problems. We finally present Monte Carlo simulations demonstrating superior performance of our algorithm compared to existing methods.


\section{Problem Setting}
First, we discuss standard k-armed MAB problem and then move to k-armed dueling bandit problem.
\subsection{Multi-Armed Bandit model}
We consider a stochastic multi-armed bandit problem with $K$ arms such that $K$ is finite and $K \geq 2$. We define $\mathcal{A} = \{1,2,...,K\}$ as set of arms. Time proceeds in round indexed by $n = 1,2,...,T$. In each round, reward $X_a(n)$ is received for playing arm $\pi_n = a \in \mathcal{A}$. These rewards are bounded in $\Theta = [0,1]$. Sequences $\big(X_a(n)\big)_{n \geq 1}$ for all arms $a$ are i.i.d with common expectation $\mu_a$.  Rewards across arms are also assumed to be independent. We denote $N_a(n)$ as number of times arm $a$ was played till round $n$, i.e. $N_a(n) = \sum_{t=1}^n \mathds{1}_{\pi_t = a}$.  At each round, a decision rule or algorithm plays an arm depending on past decisions and rewards observed. The set $\Pi$ of all possible decision rules consists of policies $\pi$ such that event $\{(\pi_n = a)\}$ (play arm $a$ at round $n$) belongs to $\sigma$ field $\mathcal{F}_{n-1}$ generated by $\pi_1,X_{\pi_1},\pi_2,X_{\pi_2},...,\pi_{n-1},X_{\pi_{n-1}}$. We denote $a^*$ as the best arm and $\mu^* = \mu_{a^*}$ as expected reward associated with it. Regret and expected regret for a policy $\pi$ at round $n$ is defined as:
\begin{equation}
    R^{MAB}_{\pi}(n) = \sum\limits_{k \in \mathcal{A}} (\mu^* - X_a^{\pi}(n)) \mathds{1}_{(\pi_n = k)},
\end{equation}
\begin{equation}
    \mathds{E}[R^{MAB}_{\pi}(n)] = \sum\limits_{k \in \mathcal{A}} (\mu^* - \mu_k) \mathds{E}[\pi_n = a],
\end{equation}
where $\mathds{1}(\cdot)$ is indicator function. Performance of any decision rule is measured by expected cumulative regret. Expected cumulative regret till round $T$ for policy $\pi$ is given by:
\begin{equation}
    \mathds{E}[R^{MAB}_{\pi,T}] = \sum\limits_{k \in \mathcal{A}} (\mu^* - \mu_k) \mathds{E}[N_a^{\pi}(T)].
\end{equation}
Any MAB algorithm aims to find a policy $\pi^*$ that minimises regret, formally:
\begin{equation}
    \pi_{MAB}^* = \underset{\pi \in \Pi}{\arg\min} \Big(\limsup_{T \rightarrow \infty} \frac{\mathds{E}[R^{MAB}_{\pi,T}]}{\log(T)} \Big).
\end{equation}
Bernoulli Kullback-Leibler divergence for $(p,q) \in \Theta^2$ as mentioned earlier is defined as:
\begin{equation}
    \mathcal{K}(p,q) = p\log(\tfrac{p}{q}) + (1-p)\log(\tfrac{1-p}{1-q}),
\end{equation}
with, by convention, $0\log0 = 0\log\tfrac{0}{0}=0$ and $x\log \frac{x}{0} = +\infty$ for $x > 0$.

\subsection{Dueling Bandits Problem}
We consider a $K$ armed dueling bandit problem such that $K \geq 2$ and finite. We define $\mathcal{A} = \{1,2,...,K\}$ as set of arms. Time proceeds in round indexed by $n = 1,2,...,T$. In each round $n$, an arm pair $\pi_n = (a_1(n),a_2(n)) \in \mathcal{A}^2$ is played and a noisy comparison $w_n\big((a_1(n),a_2(n))\big)$ is obtained. If arm $a_1(n)$ was preferred over arm $a_2(n)$ then $w_n = 1$ else $w_n = 0$. This comparison is characterised by a preference matrix $P = [P_{ij}]_{K \times K}$, where $P_{ij}$ is probability of arm $i$ being preferred over arm $j$. We assume comparisons are independent and remain stationary over time. Also we assume order of comparison does not affect outcome i.e. $(i,j)$ and $(j,i)$ would lead to same outcome. Thus $P_{ij} = 1 - P_{ji}$. We assign $P_{ii} = 0.5$. When we say arm $i$ beats $j$, we mean $P_{ij} > 0.5$. \\

\subsubsection{Copeland dueling bandits}
Condorcet winner may not often exist in practise. One of the straight forward way to declare a winner in such cases is the player or action which secures maximum wins. For example, a football team winning a league has not necessarily defeated all other teams but has defeated maximum number of teams. Copeland winner in dueling bandit problem is an arm which defeats maximum number of arms. Copeland score for any arm $i$ is defined as $\zeta_i = \frac{1}{K-1} \sum_{j \in \mathcal{A}} \mathds{1}_{(P_{ij} > 0.5)}$. Thus an arm with maximum Copeland score is Copeland winner. Formally, we say arm an $a^*$ is Copeland winner if $a^* = \arg\max_{i \in \mathcal{A}} \;\zeta_i$. Now, let us assume arm $a^*$ is Copeland winner, then regret at round $n$ for policy $\pi_n = (a_1(n),a_2(n))$ is defined as:
\begin{equation}
    R^{Cope}_{\pi}(n) = \dfrac{2\zeta_{a^*} - \zeta_{a_1(n)} - \zeta_{a_2(n)}}{2}.
\end{equation}
Throughout our discussion we assumed existence of unique winner. Like in any standard MAB algorithm, any decision rule plays an arm pair $(a_1(n),a_2(n))$ depending on previously played arm pair and observed rewards. Any decision policy $\pi$ such that event $\{(\pi_n = (a_1(n),a_2(n)))\}$ belongs to $\sigma$ field $\mathcal{G}_{n-1}$ generated by \\ $\pi_1,w_1(\pi_1),\pi_2,w_2(\pi_2),...,\pi_{n-1},w_{n-1}(\pi_{n-1})$.
As in MAB problem, any k-armed dueling bandit algorithm aims to find a policy $\pi^*$ which minimises the cumulative regret, formally:
\begin{equation}
    \pi_{U}^* = \underset{\pi \in \Pi}{\arg\min} \Big(\limsup_{T \rightarrow \infty} \frac{\mathds{E}[\sum_{n=1}^T R^{U}_{\pi}(n)]}{\log(T)} \Big).
\end{equation}
where $U$ can be any type of problem like Copeland, Condorcet, Borda etc. In next section we discuss our Sup-KLUCB algorithm to solve k-armed dueling bandit problem by converting it into standard MAB problem.

\section{Sup-KLUCB Algorithm}
We now introduce Sup-KLUCB algorithm which is applicable to any $k$ armed bandit problem with a single winner. We first define few notations required in our algorithm. 
\\
We define $Sup_i$ to be any score for arm $i$ which is used to define a single winner i.e.,  for Borda problems $Sup_i$ is Borda score for $i^{th}$ arm, for Condorcet and Copeland problems $Sup_i$ is Copeland score $\zeta_i$ for $i^{th}$ arm. In other words Sup(short for Superior) is a measure to rank various arms based on any fixed criteria. We require $Sup_i \in \Theta = [0,1]$, if it is in some general interval $[a,b]$, it can be normalised to $[0,1]$. We have assumed that our problem only has a single winner, i.e. assuming w.l.o.g. that arm 1 is winner, i.e. $Sup^* = Sup_1 = \max_{i \in \mathcal{A}} Sup_i$ and $Sup^* - Sup_i = 0$ only when $i = 1$.
\\
Let us define $\mathcal{B} = \{(i,j): i \leq j; i,j \in \mathcal{A}\}$ with cardinality $\mathcal{B}$ as $|\mathcal{B}| = \Bar{K} := \frac{K(K+1)}{2}$, where $|\cdot|$ is cardinality of a set. We define a bijective function $f:\mathcal{B} \rightarrow \mathcal{C}$ where $\mathcal{C} = \{1,2,...,\Bar{K}\}$. Note that we have allowed any bijective function without considering the exact order of mapping between elements in $\mathcal{B}$ and in $\mathcal{C}$.
We denote inverse of function $f$ by $f^{-1}:\mathcal{C} \rightarrow \mathcal{B}$. Now for $f((i,j)) = k$, we define  $\mu_k = Sup_i \cdot Sup_j$, $\mu_k \in \Theta$, where $a \cdot b$ represent scalar product of $a$ and $b$.
\\
Sup-KLUCB converts $K$-armed dueling bandit problem to single MAB problem with $\Bar{K}$ arms. Each arm pair in $\mathcal{B}$ is considered as a single arm with expected mean $\mu$. Note for $i \neq j$, if $(i,j) \in \mathcal{B}$ then $(j,i) \notin \mathcal{B}$. This is because including arm pair $(j,i)$ will not bring any new information as $P_{ij} = 1 - P_{ji}$. One can argue that arm pair $(i,i)$ also does not bring any new information as $P_{ii} = 0.5$ is fixed, but any decision making rule like RUCB or DTS after enough exploration plays $(1,1)$ (say arm $1$ is winner). We are now ready to explain Sup-KLUCB algorithm.
\\
We use following notations: $T$ is horizon, $n$ is round index,  $c_1,c_2$ are hyper parameters. For $f^{-1}(k) = (i,j)$, $N(k)$ is the number of times arm pair $(i,j)$ has been played, $W(k)$ is the number of times arm $i$ won over arm $j$. $w_n(k)$ is reward for played arm pair $(i,j)$ at each round $n$. We denote our selected arm in each round as $a \in \mathcal{C}$.
\\
After giving a broad picture of algorithm and defining several notations, we now show that for any $k$-armed bandit problem, any competent algorithm which declares an arm (say $a$) as winner is also a winner by Sup-KLUCB algorithm.  
\\
\begin{theorem}
\label{thm1}
Given there exists a unique winner, an arm $a^*$ is winner of Copeland $k$-armed dueling bandit problem if and only if it is also winner by Sup-KLUCB algorithm.
\end{theorem}

\begin{proof}
Winner is unique i.e. $\zeta_i = \zeta_j$ if and only if $i = j$ for all $i,j \in \mathcal{A}$. If an arm $a^*$ is the winner of $k$-armed dueling bandit problem then $a^* = \arg\max_{i \in \mathcal{A}} \;\zeta_i$ and for Copeland problem $Sup_i = \zeta_i$. Since $Sup_i \in \Theta \geq 0$ for all $i \in \mathcal{A}$, we have $Sup_{a^*} \cdot Sup_{a^*} = \max_{i,j \in \mathcal{A}} Sup_{i} \cdot Sup_{j}$. Thus we have an arm $k^* \in \mathcal{C}$ such that $f((a^*,a^*)) = k^*$ and $k^* = \arg\max_{k \in \mathcal{C}} \;\mu_k$. For any stochastic bounded MAB problem, KLUCB has been proved to declare true winner asymptotically. Thus for our problem, considering all $\bar{K}$ arms in $\mathcal{C}$ to be independent, KLUCB will declare the winner with maximum reward $\mu$ i.e. $k^*$. Thus a winner in $k$-armed dueling bandit problem is also winner in Sup-KLUCB algorithm. Now we prove the other way by contraposition argument. By our uniqueness assumption, if arm $b \in \mathcal{A}$ is not a winner, then $Sup_b \neq \max_{i \in \mathcal{A}} Sup_i$ and thus for an arm $l \in \mathcal{C}$ such that $f((b,b)) = l$, $\mu_l \neq \max_{k \in \mathcal{C}} \mu_k$. Hence $l$ will not be declared as winner by KLUCB being sub-optimal. This concludes our proof.
\end{proof}

We define a function $h:\mathcal{A}\rightarrow [\alpha,\beta]$ where $\alpha \geq 0; \beta > 0; \alpha < \beta; \alpha, \beta \in \mathds{R}$ that defines a unique winner $a^*$ such that $a^* = \arg\max_{i \in \mathcal{A}} h(i)$ (like Copeland winner, Borda winner etc.) or $a^* = \arg\min_{i \in \mathcal{A}} h(i)$ (when criteria is arm with minimum number of losses).

\begin{lemma}
Given any function (winner criteria) $h$ as defined above, an arm $a^*$ is winner of $k$-armed dueling bandit problem if and only if it is also winner by Sup-KLUCB algorithm.
\end{lemma}
\begin{proof}
If function $h$ is such that winner acquires maximum value i.e., $a^* = \arg\max_{i \in \mathcal{A}} h(i)$, then for any arm $i$ we can define $Sup_i = \frac{h(i) - \alpha}{\beta - \alpha}$. Now we have $Sup_{a^*} = \max_{i \in \mathcal{A}} Sup_{i}$. Hence proved from Theorem 1. If function $h$ is such that winner acquires minimum value i.e., $a^* = \arg\min_{i \in \mathcal{A}} h(i)$, then for any arm $i$ we can define $Sup_i = \frac{\beta - h(i)}{\alpha - \beta}$. Again we have $Sup_{a^*} = \max_{i \in \mathcal{A}} Sup_{i}$. Hence proved from Theorem 1.
\end{proof}

We now explain Sup-KLUCB algorithm. In Algorithm 1, for Copeland problems, we experimentally found $c_1 = \tfrac{2}{K}$ and $c_2 = \tfrac{3}{K} + \tfrac{40}{(K-2)^2}$. Now we explain each step of algorithm 1. Steps 1-4 are run once for each arm. In step 6 and 7, we calculate $Sup_i$ for all $i \in \mathcal{B}$ and $\mu_k$ for all $k \in \mathcal{C}$ respectively. In step 8, we select arm $a \in \mathcal{C}$ with highest upper confidence calculated using KL divergence $\mathcal{K}$. Arm pair $a$ is played and we declare winner at round $T$ with the arm having maximum $Sup$ value.

\begin{algorithm}[!ht]
\caption{Sup-KLUCB Algorithm}\label{Copeland}
\begin{algorithmic}[1]
\Require $T$, $c_1$, $c_2$

\For{n = 1 to $\bar{K}$}
    \State $N[n] \gets 1$
    \State $W[n] \gets w(n)$  
\EndFor

\For{n = $\bar{K}$ + 1 to T}
    \State Calculate $Sup_i$ for all $i \in \mathcal{B}$
    \State Calculate $\mu_k$ for all $k \in \mathcal{C}$
    \State $a \gets \arg\max_{a \in \mathcal{C}} \max \Big\{q \in \Theta : \, N[a] \mathcal{K}\big(\mu_a,q\big) \leq c_1 \ln(n-\Bar{K}) + c_2 \ln(\ln(n - \Bar{K}) + 1)  \Big\}$
    \State $N[a] \gets N[a] + 1$
    \State $W[a] \gets w_n(a)$
\EndFor \Comment{Winner has maximum $Sup_i$}
\end{algorithmic}
\end{algorithm}


\section{Experiments}
We performed Monte Carlo simulations to prove the performance of our algorithm. For Copeland problem, we have compared our algorithm with state of art DTS algorithm and RUCB algorithm because it is an UCB based algorithm. For our Monte Carlo experiment, we randomly chose number of arms from 3 to 36. We played 25 games with 25 iteration each with each game played upto 100,000 time steps. Preference or Probability matrix was generated randomly. We only have one assumption that winner must be unique. In figure below, we show average cumulative regret of different algorithms. We only show $50\%$ confidence interval as higher number would engulf almost whole graph.

 \begin{figure}[!ht]%
     \centering
     \subfloat[Cumulative regret of RUCB, DTS and Sup-KLUCB for Copeland bandit problem, number of arms varies from 3 to 36]{{\includegraphics[width=0.47\linewidth, height= 0.3525\linewidth]{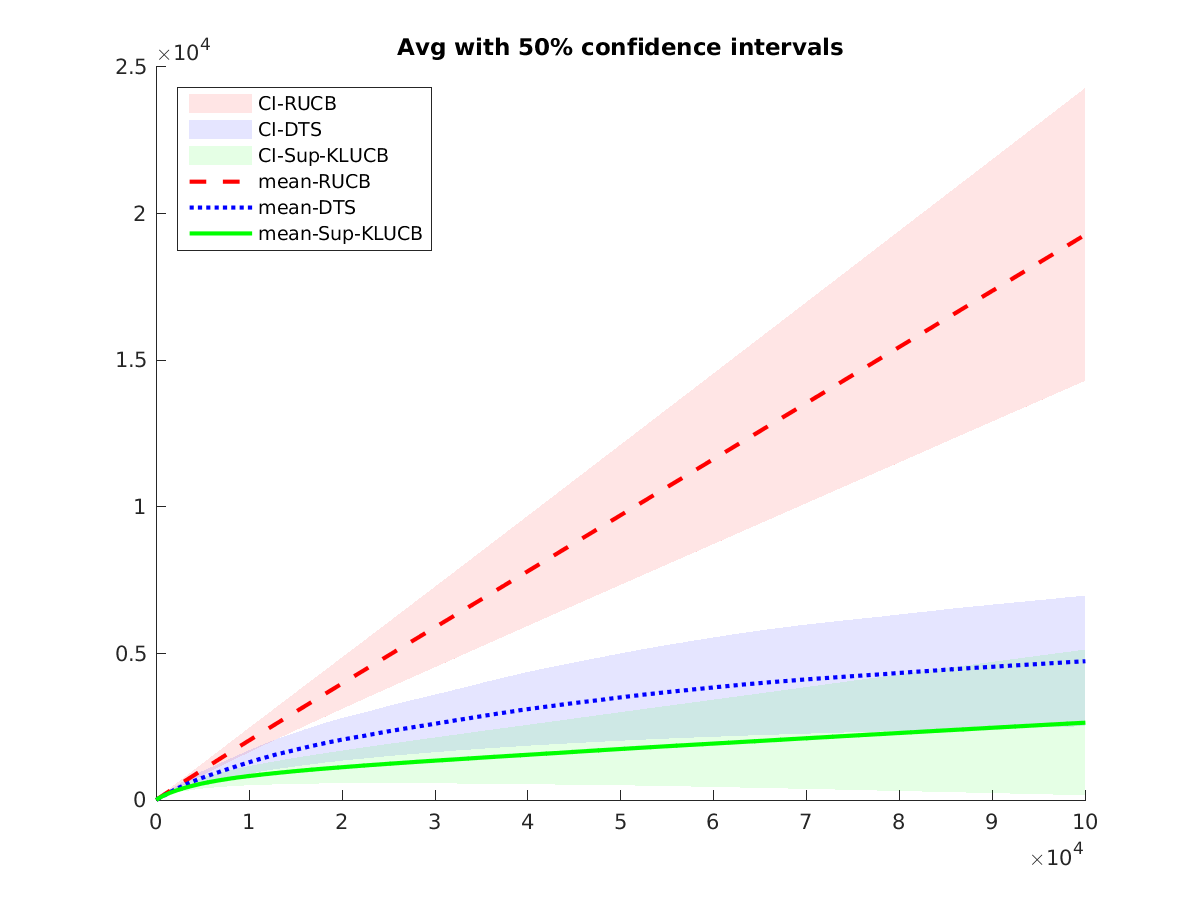}}}%
     \qquad
     \subfloat[Average regret at time step 100,000 of RUCB, DTS and Sup-KLUCB for different number of arms]{{\includegraphics[width=0.47\linewidth, height= 0.3525\linewidth]{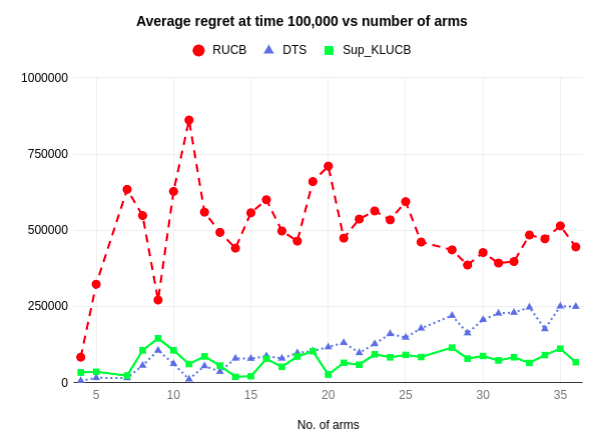} }}%
 \end{figure}





In figure (a) and (b) dashed line is for RUCB, dotted line is for current state of the art DTS and solid line is for our algorithm Sup-KLUCB. From the figure (a), we can infer that Sup-KLUCB outperform RUCB and DTS algorithm. \\
Further, we analysed performance with respect to number of arms. For this we simulated for 100 different games, each iterated 25 times up to 100,000 time steps. Figure (b)
shows that performance of Sup-KLUCB with respect to DTS increases as we increase number of arms and for smaller number of arms Sup-KLUCB performs at par with DTS.
\section{Conclusion}
In this paper we proposed a general framework for conversion of dueling bandit problem to MAB problem. This framework has very wide scope of application in terms of type of problem,and with minor changes in objective function it can used for variety of problems like Copeland, Condorcet, and Borda. Using KLUCB, we further extended UCB to dueling bandits and outperformed current state of the art algorithm for Copeland bandit problem. \\
Our future works includes detailed mathematical analysis of regret and upper and lower bounds and using this analysis, we need to firmly establish hyper-parameters used in our algorithms. 


\newpage
\begin{thebibliography}{10}

\bibitem{Yue2012}
Yisong Yue, Josef Broder, Robert Kleinberg, and Thorsten Joachims.
\newblock The k-armed dueling bandits problem.
\newblock {\em Journal of Computer and System Sciences}, 2012.

\bibitem{yue2009interactively}
Yisong Yue and Thorsten Joachims.
\newblock Interactively optimizing information retrieval systems as a dueling
  bandits problem.
\newblock In {\em Proceedings of the 26th Annual International Conference on
  Machine Learning}, pages 1201--1208. ACM, 2009.

\bibitem{sui2014clinical}
Yanan Sui and Joel Burdick.
\newblock Clinical online recommendation with subgroup rank feedback.
\newblock In {\em Proceedings of the 8th ACM Conference on Recommender
  systems}, pages 289--292. ACM, 2014.

\bibitem{szorenyi2015online}
Bal{\'a}zs Sz{\"o}r{\'e}nyi, R{\'o}bert Busa-Fekete, Adil Paul, and Eyke
  H{\"u}llermeier.
\newblock Online rank elicitation for plackett-luce: A dueling bandits
  approach.
\newblock In {\em Advances in Neural Information Processing Systems}, pages
  604--612, 2015.

\bibitem{ZoghiCopeland}
Masrour Zoghi, Zohar~Shay Karnin, Shimon Whiteson, and Maarten de~Rijke.
\newblock Copeland dueling bandits.
\newblock {\em CoRR}, abs/1506.00312, 2015.

\bibitem{Jamieson2015}
Kevin Jamieson, Sumeet Katariya, Atul Deshpande, and Robert Nowak.
\newblock Sparse dueling bandits.
\newblock {\em AISTATS}, 2015.

\bibitem{lai1985}
Tze~Leung Lai and Herbert Robbins.
\newblock Asymptotically efficient adaptive allocation rules.
\newblock {\em Advances in applied mathematics}, 6(1):4--22, 1985.

\bibitem{graves1997asymptotically}
Todd~L Graves and Tze~Leung Lai.
\newblock Asymptotically efficient adaptive choice of control laws incontrolled
  markov chains.
\newblock {\em SIAM journal on control and optimization}, 35(3):715--743, 1997.

\bibitem{auer2002finite}
Peter Auer, Nicolo Cesa-Bianchi, and Paul Fischer.
\newblock Finite-time analysis of the multiarmed bandit problem.
\newblock {\em Machine learning}, 47(2-3):235--256, 2002.

\bibitem{garivier2011kl}
Aur{\'e}lien Garivier and Olivier Capp{\'e}.
\newblock The kl-ucb algorithm for bounded stochastic bandits and beyond.
\newblock In {\em Proceedings of the 24th annual Conference On Learning
  Theory}, pages 359--376, 2011.

\bibitem{sui2018advancements}
Yanan Sui, Masrour Zoghi, Katja Hofmann, and Yisong Yue.
\newblock Advancements in dueling bandits.
\newblock In {\em IJCAI}, pages 5502--5510, 2018.

\bibitem{Zoghi2014}
Masrour Zoghi, Shimon Whiteson, Remi Munos, and Maarten de~Rijke.
\newblock Relavtive upper confidence bound for the k-armed dueling bandit
  problem.
\newblock {\em ICML}, 2014.

\bibitem{wu2016double}
Huasen Wu and Xin Liu.
\newblock Double thompson sampling for dueling bandits.
\newblock In {\em Advances in Neural Information Processing Systems}, pages
  649--657, 2016.

\bibitem{ailon2014reducing}
Nir Ailon, Zohar Karnin, and Thorsten Joachims.
\newblock Reducing dueling bandits to cardinal bandits.
\newblock In {\em International Conference on Machine Learning}, pages
  856--864, 2014.

\bibitem{sui2017multi}
Yanan Sui, Vincent Zhuang, Joel~W Burdick, and Yisong Yue.
\newblock Multi-dueling bandits with dependent arms.
\newblock {\em arXiv preprint arXiv:1705.00253}, 2017.

\end{thebibliography}
\end{document}